\documentclass[12pt,onecolumn,journal]{IEEEtran}
\usepackage[utf8]{inputenc} % allow utf-8 input
\usepackage[T1]{fontenc}    % use 8-bit T1 fonts
\usepackage{hyperref}       % hyperlinks
\usepackage{url}            % simple URL typesetting
\usepackage{booktabs}       % professional-quality tables
\usepackage{amsfonts}       % blackboard math symbols
\usepackage{nicefrac}       % compact symbols for 1/2, etc.
\usepackage{microtype}      % microtypography
\usepackage{amsmath}
\usepackage{amssymb}
\usepackage{latexsym}
\usepackage{CJK}
\usepackage{cite}

\usepackage{color, soul}

\usepackage{graphicx}
\usepackage{epstopdf}
\usepackage{epsfig}
\usepackage{subfigure} % Í¼Æ¬²¢ÅÅÏÔÊ¾

\usepackage{verbatim}  % \begin{comment} ... \end{comment}

\usepackage{mathrsfs}	
\usepackage{algorithm} %format of the algorithm
\usepackage{algorithmic} %format of the algorithm
\usepackage{textcomp}
\usepackage{multirow}
\usepackage{lettrine}

\usepackage{mathrsfs}

\usepackage{algorithm}
\usepackage{algorithmic}
\usepackage{extarrows}
\usepackage{colortbl}
\definecolor{mygray}{gray}{.9}
\usepackage{amsthm} %thm / proof

\newtheorem{cor}{Corollary}
\newtheorem{lem}{Lemma}
\newtheorem{prop}{Proposition}

\newtheorem{rem}{Remark}

{ \bgroup
  \addtolength\abovedisplayshortskip{#1}
  \addtolength\abovedisplayskip{#1}
  \addtolength\belowdisplayshortskip{#1}
  \addtolength\belowdisplayskip{#1}}
{\egroup\ignorespacesafterend}

\usepackage{subeqnarray}%×Ó¹«Ê½½øÐÐ±àºÅ
\usepackage{cases}%À¨ºÅÁªÁ¢ºê°ü
\usepackage{mathtools} %%Ëõ¶Ì¹«Ê½×ÖÄ¸¼äµÄ¾àÀë
\usepackage{bm}%%¹«Ê½ÖÐ¼Ó´Öºê°ü

\title{MIM-Based GAN: Information Metric to Amplify Small Probability Events Importance in Generative Adversarial Networks}

\author{\IEEEauthorblockN{
Rui She,
Pingyi Fan, \IEEEmembership{Senior Member,~IEEE}
\\}

\thanks{
R. She and P. Fan are with Beijing National Research Center for Information Science and Technology and the Department of Electronic Engineering, Tsinghua University, Beijing, 100084, China (e-mail: sher15@mails.tsinghua.edu.cn; fpy@tsinghua.edu.cn).
}
}

\begin{document}

\maketitle

\begin{abstract}
In terms of Generative Adversarial Networks (GANs), the information metric to discriminate the generative data from the real data, lies in the key point of generation efficiency, which plays an important role in GAN-based applications, especially in anomaly detection.
As for the original GAN, there exist drawbacks for its hidden information measure based on KL divergence on rare events generation and training performance for adversarial networks.
Therefore, it is significant to investigate the metrics used in GANs to improve the generation ability as well as bring gains in the training process.
In this paper, we adopt the exponential form, referred from the information measure, i.e. MIM, to replace the logarithm form of the original GAN.
This approach is called MIM-based GAN, has better performance on networks training and rare events generation.
Specifically, we first discuss the characteristics of training process in this approach.
Moreover, we also analyze its advantages on generating rare events in theory.
In addition, we do simulations on the datasets of MNIST and ODDS to see that the MIM-based GAN achieves state-of-the-art performance on anomaly detection compared with some classical GANs.
\end{abstract}


\section{Introduction}
In order to produce deceptive generative data, GANs are proposed as a kind of efficient approach \cite{Generative-Adversarial-Nets}.
Especially, complex or high-dimensional distributions are handled pretty well by GANs
\cite{On-the-effectiveness,Capturing-joint-label-distribution}.
Actually, the core idea of GANs is to generate samples whose distribution approximates the target distribution as much as possible.
In practice, GANs are applied in many scenarios \cite{Generative-adversarial-active,Memory-augmented,
Photo-realistic-single-image,Imitating-driver-behavior,
Learning-from-simulated-and-unsupervised}, such as images reconstruction, autonomous driving models, and samples augmentation.
In theory, the framework of GANs consists of a generator network and a discriminator network, which respectively minimizes and maximizes the distinction between the real data distribution and the generated distribution.
In this case, the information distance (which is Jensen-Shannon divergence in the original GAN) plays a vital role in the generative adversarial process.
On one hand, the performance of training process depends on the information distance.
On the other hand, the efficiency of generative data (especially for rare events data) is related to this metric.
Therefore, it is worth investigating the information distance and its corresponding objective function of GANs to improve the training process and GAN-based applications such as anomaly detection.
%
%[1] Ian J. Goodfellow, Jean Pouget-Abadie, Mehdi Mirza, Bing Xu, David Warde-Farley, Sherjil Ozair, Aaron Courville, Yoshua Bengio,
%Generative Adversarial Nets.
%Advances in neural information processing systems. 2014, Dec. pp. 2672--2680.
%[2] Christian Ledig, Lucas Theis, Ferenc Huszar, Jose Caballero, Andrew P. Aitken, Alykhan Tejani, Johannes Totz, Zehan Wang, and Wenzhe Shi. Photo-realistic single image super-resolution using a generative adversarial network. CoRR, abs/1609.04802, 2016.
%[3] Alex Kuefler, Jeremy Morton, Tim Allan Wheeler, and Mykel John Kochenderfer. Imitating driver behavior with generative adversarial networks. CoRR, abs/1701.06699, 2017.
%[4] Ashish Shrivastava, Tomas Pfister, Oncel Tuzel, Josh Susskind, Wenda Wang, and Russ Webb. Learning from simulated and unsupervised images through adversarial training. CoRR, abs/1612.07828, 2016.
%

\subsection{Different information distances for GANs}
Considering the information distance choice of GANs, there are some literatures discussing how the different distances make impacts on the optimization of objective functions in the generative process.
In terms of the original GAN, it optimizes the Jensen-Shannon divergence (which is based on Kullback-Leibler (KL) divergence) to generate samples.
Similar to the original generative model, the affine projected GAN (or AffGAN for short) is also discussed to minimize the KL divergence between the two distributions (namely, the real data distribution and the generative one), which is suitable for Maximum a Posterior (MAP) inference with respect to image super-resolution \cite{Amortised-map-inference}.
However, there exists failtiness for KL divergence as a metric in the objective function of GANs, which performs in training stability and efficiency.
To make up for this, some other information distances are considered to replace the original KL divergence as follows.
% GAN && amortised map inference for image super-resolution.

As a kind of refined GAN, Energy-Based Generative Adversarial Network (EBGAN) based on the total variation, allocates lower energies to the adjacent regions of the real data manifold and larger energies to the rest regions. Compared with the original GAN, the EBGAN performs more stable in the training process \cite{Energy-based-generative}.
% EBGAN
Moreover, to overcome the vanishing gradients problem caused by the loss function of the original GAN, Least Squares Generative Adversarial Networks (LSGAN) is proposed, which is to minimize the Pearson $\chi^2$ divergence as the objective function \cite{Least-squares-generative-adversarial-networks}.
% Least Squares Generative Adversarial Networks (LSGANs)
Another distance, Chi-Squared distance, is also used to design the Fisher GAN which constrains the second order moments of the critic and leads to train the adversarial networks in a stable way \cite{Fisher-GAN}.
% Fisher GAN
To extend GANs into a general way, $f$-divergence is discussed to train generative models where the benefits of different information distances belonging to $f$-divergence are investigated with respect to the training complexity and the quality \cite{f-GAN}.
% F-GAN
In addition, Wasserstein-GAN (WGAN) is proposed to estimate Earth Mover distance (EM distance)  continuously, which overcomes the training imbalance problem of GANs.
Furthermore, WGAN-GP, a kind of refined WGAN, is introduced to enforce the Lipschitz constraint, which enables to train neural networks more stably without any hyper-parameter tuning \cite{Wasserstein-GAN,Improved-Training-of-Wasserstein-GANs}.

In brief, it is a popular research direction to introduce a promising information distance into GANs to see whether there exist performance gains on training process.
%% WGAN and WGAN-GP

\subsection{Generative efficiency for rare events}

Anomaly detection is a common problem with real-world significance
\cite{Atd,Reverse-nearest-neighbors,Adaptive-anomaly-detection}.
Inspired by the success of neural networks, GANs are considered as an efficient approach to detect anomalies which are regarded as rare events from the perspective of occurrence probability.
Particularly, the original GAN has played great roles in anomalous natural and medical images detection \cite{Anomaly-Detection-with-Generative-Adversarial-Networks,
Image-Anomaly-Detection-with}.
%%Dan Li, Dacheng Chen, Jonathan Goh, and See-Kiong Ng. ``Anomaly Detection with Generative Adversarial Networks for Multivariate Time Series¡¯¡¯
%%Lucas Deecke, Robert Vandermeulen, Lukas Ruff, Stephan Mandt, and Marius Kloft. ``Image Anomaly Detection with Generative Adversarial Networks¡¯¡¯

%%===principle=====
In terms of the anomaly detection with GANs, the method is to regard the samples as anomalies depending on if the appropriate representations of samples in the latent space of generator are found.
Specifically, the generator learns an approximate distribution of the training data, in which there are more normal events (hardly ever with anomalies).
Thus, for a normal testing sample, it is probable to find a point in the latent space of GANs similar to this one, while for an anomalous sample it is not.
%%======GANs========
Based on this idea, there exist works using GANs in anomaly detection as follows \cite{Unsupervised-anomaly-detection,Training-adversarial-discriminators,
Adversarially-Learned}.

A method called AnoGAN using normal data to train the original GAN \cite{Unsupervised-anomaly-detection}, defines an anomaly score to distinguish the generative samples (namely normal samples) and the anomalous samples.
Moreover, another similar method learns two generators by training a conditional GAN %[22]
to reconstruct normal frames (with low reconstruction loss) and rare frames (with high reconstruction loss) \cite{Image-to-image-translation}.
Besides, an unsupervised feature learning framework, Bidirectional Generative Adversarial Networks (BiGAN) is also proposed to train the networks with normal data and combine the reconstruction loss and discriminator loss as the score function to detect anomalies \cite{Adversarial-feature-learning}. %[23]
%%============================
%[8] T. Schlegl, P. Seeb¡§ock, S. M. Waldstein, U. Schmidt-Erfurth, and G. Langs, ``Unsupervised anomaly detection with generative adversarial networks to guide marker discovery,'' CoRR, vol. abs/1703.05921, 2017. [Online]. Available: http://arxiv.org/abs/1703.05921
%[9] M. Ravanbakhsh, E. Sangineto, M. Nabi, and N. Sebe, ``Training adversarial discriminators for cross-channel abnormal event detection in crowds,'' CoRR, vol. abs/1706.07680, 2017.
%[10] H. Zenati, C. S. Foo, B. Lecouat, G. Manek, and V. R. Chandrasekhar, ``Efficient gan-based anomaly detection,'' arXiv preprint arXiv:1802.06222, 2018.
%%============================
%[22] P. Isola, J.-Y. Zhu, T. Zhou, and A. A. Efros, ``Image-to-image translation with conditional adversarial networks,'' CVPR, 2017.
%[23] J. Donahue, P. Krahenbuhl, and T. Darrell, ``Adversarial feature learning,'' arXiv preprint arXiv:1605.09782, 2016.
%%=======GANs========
Furthermore, other GAN-based anomaly detection approaches such as Generative Adversarial Active Learning (GAAL) and fast AnoGAN (or f-AnoGAN) are designed by making use of the prior information during the adversarial training process \cite{Generative-adversarial-active,f-AnoGAN}.
%[mao]
% ADGAN: anomaly detection with generative adversarial networks
% Xudong Mao, Qing Li, Haoran Xie, Raymond YK Lau, Zhen Wang, and Stephen Paul Smolley. Least squares generative adversarial networks. arXiv preprint ArXiv:1611.04076, 2016.

%%======drawback====
Moreover, the discriminator to distinguish the real samples and fake ones is also applied reasonably to detecting anomalies.
In particular, once the training process is converged, the discriminator can not cope with the testing data that is unlike the training data (which corresponds to the anomalous data) \cite{Deconvolution-and-checkerboard-artifacts,
Revisiting-classifier-two-sample-tests}.
As a result, it is feasiable to combine the discriminator and generator of GANs to detect anomalies.
%Augustus Odena, Vincent Dumoulin, and Chris Olah. Deconvolution and checkerboard artifacts.
%Distill, 1(10):e3, 2016.
%David Lopez-Paz and Maxime Oquab. Revisiting classifier two-sample tests. In International Conference on Learning Representations, 2017.

However, since the rare events (including anomalies) play a small part in the whole dataset, the generative data more likely belongs to the main part of normal data rather than small probability data.
In terms of the original GAN, the proportion of rare events in the objective function is not large enough compared with that for normal events, which makes an effect on rare events generation.

In addition, as for the original GAN, there exist drawbacks in convergence rate, sensitivity and stability of training process.
Besides, there are few literatures to investigate the improvements for objective functions (related to information distances) from the perspective of rare events generation, which is also beneficial for GAN-based anomaly detection.
Therefore, we have an opportunity to investigate a new information metric to improve the original GAN-based rare events generation and anomaly detection.

%%====our work====
In this work, we introduce the idea of Message Importance Measure (MIM) into the GANs and propose a MIM-based GAN method to detect anomalies.
In this case, there are improvements in the objective function for training process and rare events generation.
Furthermore, experiments on real datasets are taken to compare our method with other classical methods.

\subsection{Contributions and organization}
Here, we provide the contributions and organization for this work.

\begin{itemize}
\item
At first, by resorting to the idea of Message Importance Measure (MIM), MIM-based GAN is designed for probability distribution learning. As well, some characteristics of this GAN are also discussed.

\item
Then, the proposed MIM-based GAN highlights the proportion of rare events in the objective function, which reveals its advantages on rare events generation in theory.

\item
At last, MIM-based GAN performs better in anomaly detection than some other GANs.
Specifically, the experiments (based on artificial dataset and real datasets) are designed to intuitively show the different performance on the results of detection.
\end{itemize}

In addition, the organization of the rest part is provided as follows.
In Section II, we propose the MIM-based GAN by resorting to the idea of MIM and discuss its major properties in the training process.
In Section III, we theoretically analyze the effect of rare events on generators and discuss the
its corresponding anomaly detection.
Section IV compare several GANs with MIM-based GAN in the experiments of anomaly detection.
At last, this work is concluded in the Section V.

\section{Model of MIM-based GAN}

\subsection{Overview of GANs}
In terms of GANs, the essential idea is to train two artificial neural networks to generate the data, which has the same distribution as the real data distribution.
To do this, the objective functions of GANs play vital roles in the two-player game for the two neural networks which are regarded as the discriminator and generator respectively.
In fact, there exists a general form of objective function optimization for GANs, which is described as
\begin{equation}\label{eq.GAN_general_optimizaition}
\begin{aligned}
    \min_{G} \max_{D} L(D,G),
\end{aligned}
\end{equation}
where $L(D,G)$ denotes the objective function given by
\begin{equation}\label{eq.GAN_general}
\begin{aligned}
    L(D,G)
    & = \mathbb{E}_{{\bm x}\sim \mathbb{P}}[f(D({\bm x}))]
    +\mathbb{E}_{{\bm z}\sim \mathbb{P}_{z}}[g(D(G({\bm z})))] \\
    & = \mathbb{E}_{{\bm x}\sim \mathbb{P}}[f(D({\bm x}))]
    +\mathbb{E}_{{\bm x}\sim \mathbb{P}_{g_{\theta}}}[g(D({\bm x}))] \text{,}
\end{aligned}
\end{equation}
in which $f(\cdot)$ and $g(\cdot)$ are functions, $D$ is the discriminator, $G$ is the generator ($D$ and $G$ are both represented by neural networks), ${\bm x}$ and ${\bm z}$ denote the input for $D$ and $G$ respectively, as well as $\mathbb{P}$, $\mathbb{P}_{g_{\theta}}$ and $\mathbb{P}_{z}$ are distributions for real data, generative data and input data of generator.

\subsection{MIM-based GAN and Corresponding Characteristics}\label{subsection.MIM-GAN}
According to the comparison for MIM and Shannon entropy \cite{Message-importance-measure-and-its-application,Differential-message-importance-measure,
Non-parametric-message-important-measure}, it is known that the exponential function used to replace the logarithmic function, has more positive impacts on the rare events processing from the viewpoint of information metrics.
Furthermore, based on the exponential function, an information distance, Message Identification (M-I) divergence is proposed to gain the greater effect of amplification on detecting outliers than Kullback-Leibler (KL) divergence (which is based on logarithmic function) \cite{Amplifying-inter-message-distance}.
As a result, exponential function with different properties from logarithmic function
makes differences on information characterization.
In this regard, there may exist potential advantages to introduce exponential function into the original GAN which incorporates the logarithmic function in the objective function.

By virtue of the essential idea of MIM and the convexity of exponential function, we have the modified objective function of GANs as follows
\begin{equation}\label{eq.LMIM}
\begin{aligned}
    L_{\text{MIM}}(D,G)
    & = \mathbb{E}_{{\bm x}\sim \mathbb{P}}[\exp(1-D({\bm x}))]
    +\mathbb{E}_{{\bm z}\sim \mathbb{P}_{z}}[\exp(D(G({\bm z})))]\\
    & = \mathbb{E}_{{\bm x}\sim \mathbb{P}}[\exp(1-D({\bm x}))]
    +\mathbb{E}_{{\bm x}\sim \mathbb{P}_{g_{\theta}}}[\exp(D({\bm x}))]\text{,}
\end{aligned}
\end{equation}
whose notations are the same as those in $L(D,G)$ (mentioned in the Eq. (\ref{eq.GAN_general})).
%$D$ and $G$ are the discriminator and generator, ${\bm x}$ and ${\bm z}$ denote the input for $D$ and $G$ respectively, as well as $\mathbb{P}$, $\mathbb{P}_{g_{\theta}}$ and $\mathbb{P}_{z}$ are distributions for real data, generative data and input data of generator.

Similar to the original GAN, the modified objection function Eq. (\ref{eq.LMIM}) also plays the two-player optimization game with respect to $D$ and $G$ as follows
\begin{equation}\label{eq.opt_LMIM}
\begin{aligned}
     \max_{G} \min_{D} L_{\text{MIM}}(D,G)\text{,}
\end{aligned}
\end{equation}
and the corresponding adversarial networks are referred to as \textit{MIM-based GAN}.
Essentially, the goal of the MIM-based GAN is to train an optimal couple of discriminator and generator to learn the real data distribution, which is as same as that in the original GAN.
In particular, the principle of MIM-based GAN are detailed as follows.

On one hand, the network of discriminator $D$ is designed to assign the real label (ususally ``$1$'') and fake label (usually ``$0$'') to the real data and generative data (from the generator $G$).
In this case, the input pairs for $D$ consist of data (namely real data and generative data) and labels (containing real labels and fake labels).
On the other hand, the network of generator $G$ tends to output imitative data more like the real data, whose goal is to deceive the discriminator.
In other words, the discriminator $D$ is mislead by the generator $G$ to make the similar decision for the generative data and real data.
In this case, the input pairs of $G$ are data ${\bm z}$ (randomly drawn from a latent space) and real labels.
Furthermore, the loss functions of $D$ and $G$ are given by Eq. (\ref{eq.opt_LMIM}) which leads the two networks to update their weight parameters by means of back propagation.
Finally, by selecting neural networks structures for $D$ and $G$, the training process of MIM-based GAN is completed.

Then, some fundamental characteristics of MIM-based GAN are discussed as follows.

\subsubsection{Optimality of $\mathbb{P}=\mathbb{P}_{g_{\theta}}$}\label{section.opt}\ \par
Given a generator $G$, we investigate the optimal discriminator $D$ as follows.
\begin{lem}\label{lem.Optimality}
For a fixed generator $g_{\bm \theta}$ in the MIM-based GAN, the optimal discriminator $D$ is obtained as
\begin{equation}\label{eq.D*}
\begin{aligned}
    D^{*}_{\text{\rm MIM}}({\bm x})
    =\frac{1}{2}+\frac{1}{2}\ln\frac{P({\bm x})}{P_{g_{\theta}}({\bm x})}\text{,}
\end{aligned}
\end{equation}
where $P$ and $P_{g_{\theta}}$ are densities of the distributions $\mathbb{P}$ and $\mathbb{P}_{g_{\theta}}$.
\end{lem}
\begin{proof}
%Please see the Appendix \ref{app.lem.Optimality}.
Considering the training criterion of discriminator $D$ (with a given generator $g_{\bm \theta}$), we just minimize
\begin{equation}
\begin{aligned}
    L_{\text{MIM}}(D,G)
    & = \int_{{\bm x}}[ P({\bm x})\exp(1-D({\bm x})) + P_{g_{\theta}}({\bm x})\exp(D({\bm x}))] {\rm d}{\bm x}\text{,}
\end{aligned}
\end{equation}
and the corresponding solution is the optimal $D$.

As for a function $f(u)=a \exp(1-u)+b\exp(u)$ $(a>0, b>0)$, we have the solution $u=\frac{1}{2}+\ln(\frac{a}{b})$ achieving $\frac{\partial f(u)}{\partial u}=-a\exp(1-u)+b\exp(u)=0$. Besides, due to the fact that the second order derivative
$\frac{\partial^2 f(u)}{\partial u^2}=a\exp(1-u)+b\exp(u)>0$, implying the convexity of $f(u)$, the solution $u=\frac{1}{2}+\ln(\frac{a}{b})$ is obtained, which achieves the minimum of $f(u)$.
Then, it is easy to verify this Lemma.
\end{proof}

By substituting $D^{*}_{\text{MIM}}$ into $L_{\text{MIM}}(D,G)$, it is not difficult to have
\begin{equation}\label{eq.L_Doptimal}
\begin{aligned}
    & L_{\text{MIM}}(D=D^{*}_{\text{MIM}},G)\\
    & =  \mathbb{E}_{{\bm x}\sim \mathbb{P}}[\exp(1-D^{*}_{\text{MIM}}({\bm x}))]
    +\mathbb{E}_{{\bm x}\sim \mathbb{P}_{g_{\theta}}}[\exp(D^{*}_{\text{MIM}}({\bm x}))]\\
    & = \mathbb{E}_{{\bm x}\sim \mathbb{P}}
    \Big[\exp\Big(\frac{1}{2}+\ln \Big(\frac{P({\bm x})}{P_{g_{\theta}}({\bm x})}\Big)^{-\frac{1}{2}}\Big)\Big]
    + \mathbb{E}_{{\bm x}\sim \mathbb{P}_{g_{\theta}}}
    \Big[\exp\Big(\frac{1}{2}+\ln \Big(\frac{P({\bm x})}{P_{g_{\theta}}({\bm x})}\Big)^{\frac{1}{2}}\Big)\Big]\\
    & = \sqrt{{\rm e}} \bigg\{
     \mathbb{E}_{{\bm x}\sim \mathbb{P}}\bigg[\bigg(\frac{P({\bm x})}{P_{g_{\theta}}({\bm x})}\bigg)^{-\frac{1}{2}}\bigg]
    +\mathbb{E}_{{\bm x}\sim \mathbb{P}_{g_{\theta}}}\bigg[\bigg(\frac{P_{g_{\theta}}({\bm x})}{P({\bm x})}\bigg)^{-\frac{1}{2}}\bigg]
    \bigg\}\text{.}
\end{aligned}
\end{equation}

\begin{prop}\label{prop.maximum}
%According to \cite{Generative-Adversarial-Nets},
%if and only if $\mathbb{P}=\mathbb{P}_{g_{\theta}}$, the minimum objective function with the optimal discriminator is achieved as $-\ln 4$.
As for the MIM-based GAN, the optimal solution of equivalent objective function with the optimal discriminator, i.e. $L_{\text{\rm MIM}}(D=D^{*}_{\text{\rm MIM}},G)$ (mentioned in Eq. (\ref{eq.L_Doptimal})),
is achieved if and only if $\mathbb{P}=\mathbb{P}_{g_{\theta}}$,
where $L_{\text{\rm MIM}}(D=D^{*}_{\text{\rm MIM}},G)$ reaches $2\sqrt{{\rm e}}$.
\end{prop}
\begin{proof}
%Please see the Appendix \ref{app.prop.maximum}.
In the case $\mathbb{P}=\mathbb{P}_{g_{\theta}}$ (implying $D^{*}_{\text{MIM}}({\bm x})=\frac{1}{2}$),
we have the value of Eq. (\ref{eq.L_Doptimal}) as
$L_{\text{MIM}}(D=\frac{1}{2},G)=\sqrt{{\rm e}}(1+1)=2\sqrt{{\rm e}}$.
This is the maximum value of $L_{\text{MIM}}(D=D^{*}_{\text{MIM}},G)$, reached at the point $\mathbb{P}=\mathbb{P}_{g_{\theta}}$. According to the expression of Eq. (\ref{eq.L_Doptimal}), we have the equivalent formulation as follows
\begin{equation}
\begin{aligned}
    & \max_{g_{\bm\theta}} L_{\text{MIM}}(D=D^{*}_{\text{MIM}},G)\\
    & \Leftrightarrow \max_{g_{\bm\theta}}
    \sqrt{{\rm e}} \bigg\{
     \ln \mathbb{E}_{{\bm x}\sim \mathbb{P}}
     \bigg[\bigg(\frac{P({\bm x})}{P_{g_{\theta}}({\bm x})}\bigg)^{-\frac{1}{2}}\bigg]
    + \ln \mathbb{E}_{{\bm x}\sim \mathbb{P}_{g_{\theta}}}
    \bigg[\bigg(\frac{P_{g_{\theta}}({\bm x})}{P({\bm x})}\bigg)^{-\frac{1}{2}}\bigg]
    \bigg\}\text{.}
\end{aligned}
\end{equation}
Then, it is not difficult to see that
\begin{equation}
\begin{aligned}
     \max_{g_{\bm\theta}} L_{\text{MIM}}(D=D^{*}_{\text{MIM}},G)
    & \Leftrightarrow \min_{g_{\bm\theta}}
    \frac{\sqrt{\rm e}}{2}
    \bigg\{ R_{\alpha=\frac{1}{2}}(\mathbb{P}||\mathbb{P}_{g_{\theta}})
        + R_{\alpha=\frac{1}{2}}(\mathbb{P}_{g_{\theta}}||\mathbb{P})  \bigg\}\text{,}
\end{aligned}
\end{equation}
where $R_{\alpha=\frac{1}{2}}(\cdot)$ is the Renyi divergence (whose parameter satisfies $\alpha=\frac{1}{2}$) which is defined as
\begin{equation}
\begin{aligned}
    R_{\alpha}(\mathbb{P}||\mathbb{Q})
    =\frac{1}{\alpha-1} \ln \bigg\{\mathbb{E}_{{\bm x}\sim \mathbb{P}} \bigg[\bigg(\frac{P({\bm x})}{Q({\bm x})}\bigg)^{\alpha-1}\bigg] \bigg\}\text{,}
    \qquad (\alpha >0, \alpha \ne 1)\text{.}
\end{aligned}
\end{equation}
Due to the fact that Renyi divergence (with parameter $\alpha=\frac{1}{2}$) reaches the minimum when $\mathbb{P}=\mathbb{Q}$, it is readily seen that $2\sqrt{{\rm e}}$ is the maximum value of $L_{\text{MIM}}(D=D^{*}_{\text{MIM}},G)$ and
the corresponding solution is $\mathbb{P}=\mathbb{P}_{g_{\theta}}$, that is to say, the real data is replicated.
\end{proof}

\begin{rem}
Similar to the original GAN, there exists a training equilibrium for the two-player optimization game in the MIM-based GAN.
In this regard, the global optimality lies at the point of $\mathbb{P}=\mathbb{P}_{g_{\theta}}$.
However, due to the different expressions of optimization game, the training process (to the equilibrium point) for the MIM-based GAN is not the same as that for the original GAN, which is revealed by Eq. (\ref{eq.D*}) and Eq. (\ref{eq.L_Doptimal}).
Therefore, the differences between the two GANs may bring some novelties in theory and applications.
\end{rem}

\subsubsection{Gradient of generator under the optimal discriminator}\label{section.gradient}
\ \par

%In order to train networks to reach the equilibrium point faster, we often consider the convergence performance of objective function.
%In this regard, the gradient reflects the convergence rate in some degree.
With regard to analyze the training process of GANs,
%when the training state is close to the equilibrium, the convergence rate will slow down.
%To analyze convergence of GANs near the equilibrium,
it is worth investigating the gradient of objective function.
Here, we would like to discuss the gradient of generator in the case of the optimal discriminator, due to the fact that if the equilibrium is approximated, the discriminator is trained well already.

\begin{prop}\label{prop.LMIM_gradient}
As for the MIM-based GAN, let $g_{\bm \theta} : \mathcal{Z} \to \mathcal{X}$ be a differentiable function to generate data.
If there exists the optimal discriminator $D^{*}_{\text{MIM}}({\bm x})=\frac{1}{2}+\frac{1}{2}\ln\frac{P({\bm x})}{P_{g_{\theta}}({\bm x})}$,
the gradient with respect to the parameter ${\bm \theta}$ in the corresponding generator, is given by
\begin{equation}\label{eq.LMIM_gradient}
\begin{aligned}
    & \nabla_{\bm\theta} \mathbb{E}_{{\bm z}\sim \mathbb{P}_z} [\exp(D^{*}_{\text{MIM}}(g_{\bm\theta}({\bm z}))) ]\\
    & = \mathbb{E}_{{\bm z}\sim \mathbb{P}_z}
    \Big[ \nabla_{\bm\theta} \exp\Big(\frac{1}{2}+\frac{1}{2}\ln \Big(\frac{P(g_{\bm\theta}({\bm z}))}{P_{g_{\theta}}(g_{\bm\theta}({\bm z}))}\Big)\Big) \Big]\\
    & = \sqrt{{\rm e}} \mathbb{E}_{{\bm z}\sim \mathbb{P}_z}
    \Big[ \nabla_{\bm\theta}  \Big(\frac{P(g_{\bm\theta}({\bm z}))}{P_{g_{\theta}}(g_{\bm\theta}({\bm z}))}\Big)^{\frac{1}{2}} \Big]\\
    & = \frac{\sqrt{{\rm e}}}{2} \mathbb{E}_{{\bm z}\sim \mathbb{P}_z}
    \Big[  \Big(\frac{P(g_{\bm\theta}({\bm z}))}{P_{g_{\theta}}(g_{\bm\theta}({\bm z}))}\Big)^{-\frac{1}{2}}
    \frac{\nabla_{\bm\theta}P(g_{\bm\theta}({\bm z})) P_{g_{\theta}}(g_{\bm\theta}({\bm z}))
    - \nabla_{\bm\theta}P_{g_{\theta}}(g_{\bm\theta}({\bm z})) P(g_{\bm\theta}({\bm z})) }
    {P_{g_{\theta}}^2(g_{\bm\theta}({\bm z}))}
     \Big]\\
    & = \frac{\sqrt{{\rm e}}}{2} \mathbb{E}_{{\bm z}\sim \mathbb{P}_z}
    \Big[
    \frac{ \nabla_{\bm\theta}P(g_{\bm\theta}({\bm z}))
    \sqrt{\frac{P_{g_{\theta}}(g_{\bm\theta}({\bm z}))}{P(g_{\bm\theta}({\bm z}))}}
    - \nabla_{\bm\theta}P_{g_{\theta}}(g_{\bm\theta}({\bm z}))
    \sqrt{\frac{P(g_{\bm\theta}({\bm z}))}{P_{g_{\theta}}(g_{\bm\theta}({\bm z}))}} }
    {P_{g_{\theta}}(g_{\bm\theta}({\bm z}))}
    \Big]\text{.}\\\\
%    & =\sqrt{{\rm e}} \mathbb{E}_{{\bm z}\sim \mathbb{P}_z}
%    \Big[
%    \frac{ \nabla_{\bm\theta}P(g_{\bm\theta}({\bm z}))
%    - \frac{P(g_{\bm\theta}({\bm z}))}{P_{g_{\theta}}(g_{\bm\theta}({\bm z}))}
%    \nabla_{\bm\theta}P_{g_{\theta}}(g_{\bm\theta}({\bm z})) }
%    {2\sqrt{P_{g_{\theta}}(g_{\bm\theta}({\bm z})) P(g_{\bm\theta}({\bm z}))}}
%    \Big]
\end{aligned}
\end{equation}
\end{prop}

\begin{rem}
As for the gradient of generator under the optimal discriminator, it is not difficult to see if the generative distribution is approaching to the real distribution, the gradient is getting small.
In other words, if the training equilibrium is reached closely, the end of training state will be achieved.
\end{rem}

\subsubsection{Anti-interference ability of generator}\label{section.stability}\ \par
Consider the fact that the discriminator makes a difference on the generator during the adversarial training process of GANs.
Specifically, when there exists a disturbance in the discriminator, the generator will be drawn into the unstable training in some degree.
Consequently, it is required to analyze the anti-interference ability of generator with respect to the disturbance in the discriminator.
\begin{prop}\label{prop.stability}
Let $g_{\bm \theta} : \mathcal{Z} \to \mathcal{X}$ be a differentiable
function to generate the data denoted by $g_{\bm \theta}({\bm z})$, whose distribution is denoted by $\mathbb{P}_{g_{\theta}}$ corresponding to the real distribution $\mathbb{P}$. Let $\mathbb{P}_{z}$ be the distribution for ${\bm z}$ as well as $D$ be a discriminator ($D\in [0,1]$).
Consider the stability for the gradient of generator in the MIM-based GAN in the two following cases.

$\bullet$ Assuming $D-\tilde D^*=\epsilon$ ($\epsilon$ denotes a small disturbance which satisfies $\epsilon\in [0,1]$, as well as $\tilde D^*$ is the ideal perfect discriminator i.e. $\tilde D^*(g_{\bm\theta}({\bm z}))=0$), we have
\begin{equation}\label{eq.MIM_perfect_stability}
\begin{aligned}
    \nabla_{\bm\theta}\mathbb{E}_{{\bm z}\sim \mathbb{P}_z} [\exp(D(g_{\bm\theta}({\bm z}))) ]
    & = \mathbb{E}_{{\bm z}\sim \mathbb{P}_z} [\exp(D(g_{\bm\theta}({\bm z})))
    \nabla_{{\bm x}}D({\bm x})\nabla_{\bm\theta}g_{\bm\theta}({\bm z}) ]\\
    & = \mathbb{E}_{{\bm z}\sim \mathbb{P}_z} [\exp(\tilde D^*(g_{\bm\theta}({\bm z}))+\epsilon) \nabla_{{\bm x}}D({\bm x})\nabla_{\bm\theta}g_{\bm\theta}({\bm z}) ]\\
    & = \exp(\epsilon) \mathbb{E}_{{\bm z}\sim \mathbb{P}_z}
    [\nabla_{{\bm x}}D({\bm x})\nabla_{\bm\theta}g_{\bm\theta}({\bm z}) ]\text{.}
\end{aligned}
\end{equation}

$\bullet$ Assuming $D-\check D^*=\epsilon$ ($\epsilon$ is a small disturbance which satisfies $|\epsilon|<\frac{1}{2}$, as well as $\check D^*$ is the worst discriminator i.e. $\check D^*(g_{\bm\theta}({\bm z}))=\frac{1}{2}$ implying that the training equilibrium is achieved),
we have
\begin{equation}\label{eq.MIM_worst_stability}
\begin{aligned}
     \nabla_{\bm\theta}\mathbb{E}_{{\bm z}\sim \mathbb{P}_z} [\exp(D(g_{\bm\theta}({\bm z}))) ]
    & = \mathbb{E}_{{\bm z}\sim \mathbb{P}_z} [\exp(\check D^*(g_{\bm\theta}({\bm z}))+\epsilon) \nabla_{\bm x}D({\bm x}) \nabla_{\bm\theta}g_{\bm\theta}({\bm z})]\\
    & = \exp(\frac{1}{2}+\epsilon) \mathbb{E}_{{\bm z}\sim \mathbb{P}_z} [ \nabla_{\bm x}D({\bm x}) \nabla_{\bm\theta}g_{\bm\theta}({\bm z}) ]\text{.}
\end{aligned}
\end{equation}
\end{prop}
\begin{cor}\label{cor.stability}
Let $g_{\bm\theta} : \mathcal{Z} \to \mathcal{X}$ be a differentiable function that is used to generate data following the distribution $\mathbb{P}_{g_{\theta}}$.
Let $\mathbb{P}_z$ be the distribution for ${\bm z}$ (${\bm z}\in \mathcal{Z}$), $\mathbb{P}$ be the real data distribution, and $D$ be a discriminator ($D\in [0,1]$).
Consider the condition satisfying $D-\tilde D^*=\epsilon$ ($\epsilon\in [0,1]$ and $\tilde D^*(g_{\bm\theta}({\bm z}))=0$ denoting the ideal perfect discriminator) or $D-\check D^*=\epsilon$ ($|\epsilon|<\frac{1}{2}$ and $\check D^*(g_{\bm\theta}({\bm z}))=\frac{1}{2}$ denoting the worst discriminator).
In this regard, the gradient of generator in the MIM-based GAN has more anti-interference ability of generator than that in the original GAN.
\end{cor}
\begin{proof}
Please see the Appendix \ref{app.cor_stability}.
\end{proof}

\begin{rem}
From the perspective of the gradient of generator, the disturbance in the discriminator is taken
into a function to provide a multiplicative parameter for the gradient.
In fact, the different gradients in the original GAN and MIM-based GAN are resulted from the different objective functions.
Particularly, the exponential function in the gradient of MIM-based GAN is originated from the the partial derivative of its objective function with exponential one.
While, the reciprocal function in the gradient of original GAN is derived from the logarithmic function of the objective function.
\end{rem}

\section{Rare events analysis in GANs}
From a new viewpoint to analyze GANs, we focus on the case that real data contains rare events and investigate how rare events make differences on the data generation and the corresponding applications of GANs.

\subsection{Effect of rare events on generator}\label{section.rare_events_generator}

With respect to the training process of GANs, we usually train a pretty good discriminator and use it to lead the generator to reach its optimal objective function.
Then, a better generator is also obtained to train the discriminator as a feedback.
This process runs iteratively until reaches the equilibrium point.
In this regard, if we have an optimal discriminator (an ideal case), our goal is to maximize the objective function by selecting an appropriate generator (according to Eq. (\ref{eq.L_Doptimal})).
In this case, relatively fewer occurrence events make less effects on the objective function. It is implied that the generator ignores smaller probability events (usually regarded as rare events) in some degree by maximizing the major part of objective function.
As a result, it is necessary to discuss the proportion of rare events in the objective functions of generators.
Before this, we shall introduce a kind of rare events characterization to provide a specific example for rare events processing.

In general, rare events and large probability ones can be regarded to belong to two different classes, which implies there exists a binary distribution $\{ P(\bar{\Theta}), P(\Theta)\}$ where $P(\bar{\Theta})=p$ ($p<<\frac{1}{2}$) and $P(\Theta)=1-p$ ($\bar \Theta$ denotes the rare events and $\Theta$ denotes the normal ones).
For instance, the minority set and majority set match this case in statistics.
Specifically, we have the two sets satisfying
\begin{equation}\left\{
\begin{aligned}
& \bar{\Theta}=\left\{ m_k \displaystyle \big| \  | \frac{ m_k }{M}- p_k | \geq \xi, k=1,2,...,{\rm{K}} \right\}\text{,}\\
& {\Theta}=\left\{ m_k \displaystyle \big| \  | \frac{ m_k }{M}- p_k | < \xi , k=1,2,...,
{\rm{K}} \right\}\text{,}
\end{aligned}\right.
\end{equation}
and the corresponding probability elements given by
\begin{equation}\left\{
\begin{aligned}\label{eq.p_rare}
&P\{ \bar{\Theta}\} \leq K \max\limits_{k} { P\{| \frac{ m_k }{M}- p_k | \geq \xi \}} \le \delta\text{,}\\
&P\{\Theta\}=1-P\{ \bar{\Theta}\}> 1-\delta\text{,}
\end{aligned}\right.
\end{equation}
which results from the weak law of large numbers
%\begin{equation}\label{law1}
%\begin{aligned}
$P\left\{ | \frac{ m_k }{M}- p_k | < \xi \right\} > 1-\delta$,
%\end{aligned}
%\end{equation}
where $m_k$ is the occurrence number of $a_k$ (the sample support space is
$\{a_1, a_2,...,a_{\rm K}\}$),
$p_k$ denotes a probability element from a distribution \{$p_{}, p_{2},..., p_{\rm K}$\}, $M$ is the sample number, as well as $0<\xi \ll 1$ and $0<\delta \ll 1$.

Based on the above discussion, we investigate how much proportion of rare events will be taken in the objective functions of GANs.
This may reflect the rare events generation for the generator of GANs.

\begin{prop}\label{prop.proportion_rare_events_MIM}
Let $\mathbb{P}$ be the real data distribution involved with rare events, which is given by $\mathbb{P}= \{p,1-p\}$ ($0<p<<\frac{1}{2}$) corresponding to the Eq. (\ref{eq.p_rare}).
Let $g_{\bm\theta} : \mathcal{Z} \to \mathcal{X}$ be a differentiable
function for the generative data which follows the distribution $\mathbb{P}_{g_{\theta}}$ where $\mathbb{P}_{g_{\theta}}= \{ p+\varepsilon p^{\gamma}, 1-p-\varepsilon p^{\gamma} \} = \{q,1-q\}$ ($q<\frac{1}{2}$).
%similar to the relationship between $\mathbb{P}_{g_{\theta}}$ and $\mathbb{P}$ in Eq. (\ref{eq.Pg_Pr}).
%Besides, the relationship between $\mathbb{P}_{g_{\theta}}$ and $\mathbb{P}$ is described as Eq. (\ref{eq.Pg_Pr}).
Consider the case that the optimal discriminator is achieved, which implies $D^{*}_{\text{\rm MIM}}({\bm x})=\frac{1}{2}+\frac{1}{2}\ln\frac{P({\bm x})}{P_{g_{\theta}}({\bm x})}$ for the MIM-based GAN.
In this case, the proportion of rare events in the objective function of generator is given by
\begin{equation}
\begin{aligned}
    \Upsilon_{\text{MIM}}
    %& = \frac{
%    2\sqrt{{\rm e}} \bigg\{ p (1+\varepsilon p^{\gamma-1})^{\frac{1}{2}} \bigg\}
%    }{L_{\text{MIM}}(D=D^{*}_{\text{MIM}},G)}\\
    & \approx \frac{\frac{p+q}{2}-\frac{1}{8}\varepsilon^2p^{2\gamma-1}}
    { 1-\frac{1}{8} \varepsilon^2 \frac{ p^{2\gamma-1}}{1-p}}\text{,}
\end{aligned}
\end{equation}
where $\varepsilon$ and $\gamma$ denote a small disturbance parameter and an adjustable parameter (regarded as a constant) respectively.
\end{prop}
\begin{proof}
In the light of the condition $\mathbb{P}_{g_{\theta}}= \{ p+\varepsilon p^{\gamma}, 1-p-\varepsilon p^{\gamma} \} = \{q,1-q\}$ ($q<\frac{1}{2}$), it is known that $\varepsilon$ and $\gamma$ represent the deviation between the two distributions ${\mathbb{P}}$ and $\mathbb{P}_{g_{\theta}}$.
Considering the optimal discriminator in the MIM-based GAN and the Eq. (\ref{eq.L_Doptimal}), it is readily seen that
\begin{equation}
\begin{aligned}
    & L_{\text{\rm MIM}}(D=D^{*}_{\text{\rm MIM}},G)\\
    & = \sqrt{{\rm e}} \bigg\{
     \mathbb{E}_{{\bm x}\sim \mathbb{P}}\bigg[\bigg(\frac{P({\bm x})}{P_{g_{\theta}}({\bm x})}\bigg)^{-\frac{1}{2}}\bigg]
    +\mathbb{E}_{{\bm x}\sim \mathbb{P}_{g_{\theta}}}\bigg[\bigg(\frac{P_{g_{\theta}}({\bm x})}{P({\bm x})}\bigg)^{-\frac{1}{2}}\bigg]
    \bigg\}\\
    & = 2\sqrt{{\rm e}} \bigg\{
    p (1+\varepsilon p^{\gamma-1})^{\frac{1}{2}}+ (1-p)\bigg(1-\frac{\varepsilon p^{\gamma}}{1-p}\bigg)^{\frac{1}{2}}
    \bigg\}\text{.}
\end{aligned}
\end{equation}
Focusing on the rare events reflected into the probability element $p$ rather than $(1-p)$ $(0<p<<\frac{1}{2})$, we have the proportion of rare events in the
$L_{\text{MIM}}(D=D^{*}_{\text{MIM}},G)$ as follows
\begin{equation}\label{eq.rare_MIM}
\begin{aligned}
    \Upsilon_{\text{MIM}}
    & = \frac{
    2\sqrt{{\rm e}} \bigg\{ p (1+\varepsilon p^{\gamma-1})^{\frac{1}{2}} \bigg\}
    }{L_{\text{MIM}}(D=D^{*}_{\text{MIM}},G)}\\
%    {2\sqrt{{\rm e}} \bigg\{
%    p (1+\varepsilon p^{\gamma-1})^{\frac{1}{2}}+ (1-p)\bigg(1-\frac{\varepsilon p^{\gamma}}{1-p}\bigg)^{\frac{1}{2}}
%    \bigg\}}\\
    & \overset{(a)}{=} \frac{
    p+\frac{1}{2} \varepsilon p^{\gamma}-\frac{1}{8}\varepsilon^2p^{2\gamma-1}
    + o(\varepsilon^2)
    }
    { \{ p+\frac{1}{2} \varepsilon p^{\gamma} - \frac{1}{8}\varepsilon^2p^{2\gamma-1}
    + o(\varepsilon^2) \}
    + \{ (1-p)-\frac{1}{2} \varepsilon p^{\gamma}-\frac{1}{8}\varepsilon^2 \frac{p^{2\gamma}}{1-p} + o(\varepsilon^2) \}
    }\\
    & =  \frac{\frac{p+q}{2}-\frac{1}{8}\varepsilon^2p^{2\gamma-1} + o(\varepsilon^2) }
    { 1-\frac{1}{8}\varepsilon^2 \frac{p^{2\gamma-1}}{1-p} + o(\varepsilon^2)}\\
    & \approx \frac{\frac{p+q}{2}-\frac{1}{8}\varepsilon^2p^{2\gamma-1}}
    { 1-\frac{1}{8} \varepsilon^2 \frac{ p^{2\gamma-1}}{1-p}}\text{,}
\end{aligned}
\end{equation}
where the equality $(a)$ is derived from Taylor's theorem.
\end{proof}

\begin{cor}\label{cor.proportion_rare_events}
Let $\mathbb{P}$ and  $\mathbb{P}_{g_{\theta}}$ be the real distribution and the generative one,
where $\mathbb{P}= \{p,1-p\}$ ($0<p<<\frac{1}{2}$) and $\mathbb{P}_{g_{\theta}}= \{ p+\varepsilon p^{\gamma}, 1-p-\varepsilon p^{\gamma} \} = \{q,1-q\}$ ($q<\frac{1}{2}$).
Consider the case that the optimal discriminator is achieved, which implies $D^{*}_{\text{\rm MIM}}({\bm x})=\frac{1}{2}+\frac{1}{2}\ln\frac{P({\bm x})}{P_{g_{\theta}}({\bm x})}$ for the MIM-based GAN
and $D^{*}_{\text{\rm KL}}({\bm x})=\frac{P({\bm x})}{P({\bm x})+P_{g_{\theta}}({\bm x})}$ for the original GAN.
In this regard, compared with the original GAN, the MIM-based GAN usually maintains higher proportion of rare events in the objective function of generator, namely $\Upsilon_{\text{MIM}}\ge \Upsilon_{\text{KL}}$, where the equality is hold in the case $p=q$.
\end{cor}

\begin{proof}
Similar to Proposition \ref{prop.proportion_rare_events_MIM}, it is readily seen that
\begin{equation}
\begin{aligned}
    & L_{\text{KL}}(D=D^{*}_{\text{KL}},G)\\
    & = \mathbb{E}_{{\bm x}\sim \mathbb{P}}
    \bigg[\ln \frac{{P}({\bm x})}{{P}({\bm x})+ P_{g_{\theta}}({\bm x})}\bigg]
    +\mathbb{E}_{{\bm x}\sim \mathbb{P}_{g_{\theta}}}
    \bigg[\ln \frac{ P_{g_{\theta}}({\bm x})}{{P}({\bm x})+ P_{g_{\theta}}({\bm x})}\bigg]\\
    & = -p\ln (2+\varepsilon p^{\gamma-1}) - (1-p)\ln (2-\varepsilon \frac{p^{\gamma}}{1-p})\\
    &  \quad + p(1+\varepsilon p^{\gamma-1})\ln (1-\frac{1}{2+\varepsilon p^{\gamma-1}})
    + (1-p-\varepsilon p^{\gamma}) \ln (1-\frac{1}{2-\varepsilon \frac{p^{\gamma}}{1-p}})\text{,}
\end{aligned}
\end{equation}
in which the the proportion of rare events is given by
\begin{equation}\label{eq.rare_KL}
\begin{aligned}
    \Upsilon_{\text{KL}}
    &= \frac{
    -p\ln (2+\varepsilon p^{\gamma-1})
    +p(1+\varepsilon p^{\gamma-1})\ln (1-\frac{1}{2+\varepsilon p^{\gamma-1}})
    }
    {L_{\text{KL}}(D=D^{*}_{\text{KL}},G)}\\
%    & = \frac{
%    -(p+q)\ln 2 + \frac{1}{4}\varepsilon^2p^{2\gamma-1}
%    + o(\varepsilon^2)
%    }
%    { -2\ln2 + \frac{1}{4}\varepsilon^2 p^{2\gamma-1} (1+\frac{p}{1-p}) + o(\varepsilon^2)  }\\
    & %\qquad \qquad \qquad \qquad \qquad \qquad
    = \frac{\frac{p+q}{2}-\frac{1}{8\ln2}\varepsilon^2p^{2\gamma-1} + o(\varepsilon^2) }
    { 1-\frac{1}{8\ln 2}\varepsilon^2 \frac{p^{2\gamma-1}}{1-p} + o(\varepsilon^2)}\\
    &  \approx
    \frac{\frac{p+q}{2}-\frac{1}{8\ln2}\varepsilon^2p^{2\gamma-1}  }
    { 1-\frac{1}{8\ln 2}\varepsilon^2 \frac{p^{2\gamma-1}}{1-p} }\text{.}
\end{aligned}
\end{equation}
According to $p<\frac{1}{2}$ and $q<\frac{1}{2}$, we have
$ \frac{p+q}{2(1-p)}<1$. Then, it is not difficult to see that the term of right-hand side in Eq. (\ref{eq.rare_MIM}) is larger than that in Eq. (\ref{eq.rare_KL}), which verifies the proposition.
\end{proof}

\begin{rem}
Consider that the objective function guides the generator networks of GANs to generate fraudulent fake data.
Since the dominant component of objective function depends on larger probability events, a generator prefers to output the data belonging to the normal data set.
In order to reveal the ability of rare events generation for generators, it is significant to compare the rare events proportion in different objective functions, which is discussed in Proposition \ref{prop.proportion_rare_events_MIM} and Corollary \ref{cor.proportion_rare_events}.
Furthermore, according to Proposition \ref{prop.stability} and Corollary \ref{cor.stability},
it is implied that when there exists some disturbance for rare events, the MIM-based GAN still keeps more stable than the original GAN, which means the former has more anti-interference ability to generate rare events.
To sum up, it is reasonable that the MIM-based GAN performs more efficient than the original GAN on rare events generation.
\end{rem}

\subsection{GAN-based anomaly detection}
As a promising application of GANs, anomaly detection has attracted much attention of researchers.
According to the principle of GANs, it is known that by resorting to the Nash equilibrium of objective function rather than a single optimization, the generator gains more representative power and specificity to represent the real data.
Actually, to identify anomalies (belonging to rare events) with GANs, the outputs of generator network are regarded to approximate normal events.
Then, by use of identification tools such as Euclidean distance, anomalies are detected due to their evident differences from the generative events (corresponding to normal ones).
Simultaneously, a trained discriminator network is also used to dig out the anomalies not in the generative data.

Due to the similar principle between the original GAN and MIM-based GAN,
the anomaly detection method is not only suitable for the former but also for the latter.
In this regard, we will introduce how to build a data processing model based on GANs (such as the original GAN and the MIM-based GAN), and how to use it to identify anomalous events (hardly ever appearing in the training data) in details.

\subsubsection{Procedure for anomaly detection with GANs}\label{section.procedure_detection_GAN}\ \par
We shall introduce a procedure of GAN-based detection, which provides a general framework for GANs to detect anomalies (where we take the MIM-based GAN as an example).
As for our goal, it is to allocate labels $\{0,1\}$ (``$0$'' for normal events and ``$1$'' for anomalous ones) to testing samples. The procedure is described as follows.
\begin{itemize}
\item{\textbf{Step 1: data processing preparation with GANs}}

%We are given a set of M medical images $I_m$ showing healthy anatomy, with m = 1,2,...,M, where $I_m \in R^{axb}$ is an intensity image of size axb. From each image $I_m$, we extract K 2D image patches $x_k$,m of size cxc from randomly sampled positions resulting in data $x = x_k,m \in X$, with k = 1,2,...,K. During training we are only given $<I_m>$ and train a generative adversarial model to learn the manifold X, which represents the variability of the training images, in an unsupervised fashion. For testing, we are given $<y_n,l_n>$, where $y_n$ are unseen images of size cxc extracted from new testing data J and $l_n \in {0,1}$ is an array of binary image-wise ground-truth labels, with n = 1,2,...,N. These labels are only given during testing, to evaluate the anomaly detection performance based on a given pathology

At first, we generate some fake data similar to the real data by use of GAN.
%Note that we assume here that all the points in the training dataset are normal.
By feeding training data ${\bm x}$ and the random data ${\bm z}$ in the latent space to the MIM-based GAN model, we train the generator and discriminator (both based on neural networks) with the two-player maxmin game given by Eq. (\ref{eq.opt_LMIM}) to generate the imitative data.

\item{\textbf{Step 2: anomaly score computing}}

In this step, a detection measurement tool named anomaly score is designed based on the generative data and the corresponding GANs.
%% change the sentence
We trained the discriminator $D$ and generator $G$ for enough iterations and then make use of them to identify the anomalous events by means of the anomaly score which is introduced in Section $\ref{section.score}$.
%% change the sentence

\item{\textbf{Step 3: decision for detection}}

By using anomaly scores of the testing samples (obtained in the step 2) to make a decision for detection, we label each sample in the testing dataset as
\begin{equation}\label{eq.decision}
A^{\text{test}}=\left\{
\begin{aligned}
1, \quad \text{for} \quad S^{\text{test }}> \Gamma\text{,}\\
0, \quad \text{otherwise}\text{,} \qquad
\end{aligned}
\right.
\end{equation}
where $A^{\text{test}}$ is a label for a testing sample, whose non-zero value indicates a detected rare event, i.e. the anomaly score is higher than a predefined threshold $\Gamma$.
\end{itemize}

\subsubsection{Anomaly score based on both discriminator and generator}\label{section.score}
\ \par

The superiority of architecture of GANs is that we jointly train two neural networks, namely the discriminator and the generator, which makes one more decision tool available.
We would like to exploit both discriminator and generator as tools to dig out anomalies.
In fact, there are two parts in the GAN-based anomaly detection as follows.

\begin{itemize}

\item \textbf{Generator-based anomaly detection}

In terms of the trained generator $G$ which generates realistic samples, it is regarded as a mapping from a latent data space to the real data space, namely $G : \mathcal{Z} \to \mathcal{X}$.
Since that it is more likely to learn the normal data which occurs frequently,
the generator tends to reflect the principal components of the real data' distribution.
In this regard, the generator is also considered as an inexplicit model to reflect the normal events.
Considering the smooth transitions in the latent space, we have the similar outputs of  generator when the inputs are close enough in the latent space.
%% [38] (from Dan Li)
Furthermore, if we can find the latent data ${\bm z}$ which is the most likely mapped into the testing data ${\bm x}$, the similarity between testing data ${\bm x}$ and reconstructed testing data $G({\bm z})$ reveals how much extent ${\bm x}$ can be viewed as a sample drawn from the distribution reflected by $G$.
As a result, it is applicable to use the residuals between ${\bm x}$ and $G({\bm z})$ to identify the anomalous events hidden in testing data.
In addition, we should also consider the discriminator loss as a regularization for the residual loss, which ensures the reconstructed data $G({\bm z})$ to lie on the manifold $\mathcal{X}$.

\item \textbf{Discriminator-based anomaly detection}

The trained discriminator $D$ which distinguishes generative data from real data with high sensitivity, is a direct tool to detect the anomalies.

\end{itemize}

Considering GAN-based anomaly detection, it is the most important to
find the optimal ${\bm z}$ in the latent space, which is mapped to the testing samples approximately.
In this regard, a random ${\bm z}$ from the latent space is chosen and put into the generator to produce the reconstructed sample $G({\bm z})$ which corresponds to the sample ${\bm x}$.
Then, we update ${\bm z}$ in the latent space by means of gradient descent with respect to the loss function given by Eq. (\ref{eq.Lerror}).
After sufficient iteration (namely the loss function hardly ever decreasing), we gain the most likely latent data ${\bm z}$ mapped into the testing data, which means the optimal ${\bm z}_{\text{opt}}$ is obtained by
\begin{equation}
{\bm z}_{\text{opt}} = \arg \min_{{\bm z}} {J_{\text{error}}( {\bm x} , {\bm z} )}\text{,}
\end{equation}
where the loss function $J_{\text{error}}( {\bm x} , {\bm z})$ is given by
\begin{equation}\label{eq.Lerror}
{J_{\text{error}}({\bm x} , {\bm z} )} = (1-\lambda)|| {\bm x} - G({\bm z})||_{p} + \lambda H_{\text{ce}}( D(G({\bm z})), \beta)\text{,}
\end{equation}
in which $\lambda$ is an adjustable weight $(0<\lambda<1)$, $|| \cdot||_{p}$ is the $p-$norm $(\text{usually } p=2)$, and $H_{\text{ce}}( \cdot, \cdot)$ denotes the sigmoid cross entropy which is given by
\begin{equation}
\begin{aligned}
 H_{\text{ce}}( D(G({\bm z})), \beta)
 & = -\beta \ln \Big[\frac{1}{1+\exp(-D(G({\bm z})))}\Big]\\
 & \quad - (1-\beta) \ln \Big[1- \frac{1}{1+\exp(-D(G({\bm z})))}\Big]\text{,}
\end{aligned}
\end{equation}
with the target $\beta=1$.

Furthermore, by combining the ${J_{\text{error}}({\bm x} , {\bm z} )} $ and $D({\bm x})$,  we have the anomaly detection loss, referred to as {\textit{anomaly score}}, which is given by
\begin{equation}\label{eq.anomaly_score}
    S^{\text{test}} = ( 1- \eta )J_{\text{error}}({\bm x}, {\bm z}_{\text{opt}}) + \eta H_{\text{ce}}( D({\bm x}), \beta)\text{,}
\end{equation}
where the adjustable weight satisfies $0<\eta<1$, as well as, $\beta=1$.

In view of the above descriptions, the outputs of trained discriminator and generator are exploited to calculate a set of anomaly scores for test data. Then, we detect the anomalies by use of the decision-making tool described as Eq. (\ref{eq.decision}).

%By the way, We used mini-batch stochastic optimization based on Adam Optimizer and Gradient Descent Optimizer for updating the model parameters in this work.

\subsubsection{Analysis for the anomaly detection with GANs}\ \par
Here, we shall discuss the intrinsic principle of the above anomaly detection method.
Specifically, we take the detection method with the MIM-based GAN as an example to give some analyses.
Considering that the anomaly score (as a main part in the detection method) consists of two parts related to GANs, we will analyze the corresponding two parts, respectively, as follows.

\begin{itemize}
\item \textbf{Analysis for generator-based detection}:
As for the generator $G$ which maps the latent samples into realistic samples, it tends to generate the data with large probability in the real data set.
Assume the real data are classified into two sets according to their probability, namely the large probability events set $\Omega_{\text{large}}$
%(including the main normal set $\Omega_{norm_m}$ and secondary normal set $\Omega_{norm_s}$)
and the small probability events (or rare events) set $\Omega_{\text{rare}}$.
In this case, we have the proportion of large probability events in the objective function mentioned in Eq. (\ref{eq.LMIM}) as follows
\begin{equation}
\begin{aligned}
    R_{\Omega_{\text{large}}} =
    & \frac{\int_{\Omega_{\text{large}}}[ P({\bm x})\exp(1-D({\bm x})) +  P_{g_{\theta}}({\bm x})\exp(D({\bm x}))] {\rm d}{\bm x}}
    {\int_{\Omega_{\text{large}}+\Omega_{\text{rare}}}[ P({\bm x})\exp(1-D({\bm x})) +  P_{g_{\theta}}({\bm x})\exp(D({\bm x}))] {\rm d}{\bm x}}\\
    & \overset{(b)}{=}
    \frac{
     \int_{\Omega_{\text{large}}}
     \bigg[P({\bm x})\bigg(\frac{P({\bm x})}{P_{g_{\theta}}({\bm x})}\bigg)^{-\frac{1}{2}}
    + P_{g_{\theta}}({\bm x})
    \bigg(\frac{P_{g_{\theta}}({\bm x})}{P({\bm x})}\bigg)^{-\frac{1}{2}}\bigg]
    {\rm d}{\bm x}
    }
    {
     \int_{\Omega_{\text{large}}+\Omega_{\text{rare}}}
     \bigg[P({\bm x})\bigg(\frac{P({\bm x})}{P_{g_{\theta}}({\bm x})}\bigg)^{-\frac{1}{2}}
    + P_{g_{\theta}}({\bm x})
    \bigg(\frac{P_{g_{\theta}}({\bm x})}{P({\bm x})}\bigg)^{-\frac{1}{2}}\bigg]
    {\rm d}{\bm x}
    }\\
    & =
    \frac{
     \int_{\Omega_{\text{large}}}
     [{P({\bm x})}{P_{g_{\theta}}({\bm x})}]^{\frac{1}{2}} {\rm d}{\bm x}
    }
    {
     \int_{\Omega_{\text{large}}+\Omega_{\text{rare}}}
     [{P({\bm x})}{P_{g_{\theta}}({\bm x})}]^{\frac{1}{2}} {\rm d}{\bm x}
    }\text{,}
\end{aligned}
\end{equation}
where the equality $(b)$ is obtained by replacing the discriminator $D$ with $D_{\text{MIM}}^*$ (in Eq. (\ref{eq.D*})).
When the generative probability $P_{g_{\theta}}({\bm x})$ is close to the real probability $P({\bm x})$,
we have $P_{g_{\theta}}({\bm x}) \approx P({\bm x}) >>0 $ in the region $\{{\bm x} \in \Omega_{\text{large}}\}$,
while in the region $\{{\bm x} \in \Omega_{\text{rare}}\}$, $P_{g_{\theta}}({\bm x})$ is pretty small or even approximates to zero.
Then, the large probability events proportion in the objective function is close to the corresponding probability, that is, $R_{\Omega_{\text{large}}}\approx
\frac{\int_{\Omega_{\text{large}}}{P({\bm x})} {\rm d}{\bm x}}
{\int_{\Omega_{\text{large}}+\Omega_{\text{rare}}}{P({\bm x})} {\rm d}{\bm x}} = \int_{\Omega_{\text{large}}}{P({\bm x})} {\rm d}{\bm x} \to 1$.
This implies that the generative data $G({\bm z})$ is more likely to belong to the large probability events set (that is usually the principle component of normal events), regarded as the whole of normal events.
It is readily seen that large probability events make more effects on training the generator than rare events with small probability (which usually consists of the small part of normal events and anomalous ones).
Furthermore, in the loss function $J_{\text{error}}( {\bm x} , {\bm z})$,
$|| {\bm x} - G({\bm z})||_{p}$ is minimized to let ${\bm x}$ be close to a generative event (regarded as a normal event),
while $H_{\text{ce}}( D(G({\bm z})), \beta)$ enforces $G({\bm z})$ to be in the real data space.
As a result, large enough $J_{\text{error}}( {\bm x} , {\bm z})$ usually reflect the anomalous events, which plays an important role in the anomaly score.

\item \textbf{Analysis for discriminator-based detection}:
As the second term of anomaly score, $H_{\text{ce}}( D({\bm x}), \beta)$ is based on a well trained discriminator (similar to the optimal one).
In terms of the ideal discriminator as Eq. (\ref{eq.D*}), it can make a decision whether a sample belongs to the training data set (namely the real data set) or not.
In particular, when a testing data does not appear in the training data set, the corresponding value of $D({\bm x})$ approximates to zero (which implies a large cross entropy).
While, the values of $D({\bm x})$ for other testing data (included in the training data set) are more likely close to $\frac{1}{2}$.
In this regard, the discriminator is exploited for anomaly detection.
\end{itemize}

\begin{rem}
It is necessary to give comparisons for the original GAN and MIM-based GAN to detect anomalies.
Due to the similar two-player game principle, the original GAN and MIM-based GAN have analogous characteristics in the anomaly detection.
However, from the discussion of Corollary \ref{cor.proportion_rare_events}, we know that the
MIM-based GAN pays more attention to the small part of normal events (which are with smaller probability) in some degree than the original GAN.
%From the perspective of generator-based detection, effective probability support range of generative data is enlarged more in the MIM-based GAN than that in the original GAN, which helps to generate more efficient normal data.
In other words, the smaller probability events in the normal events set are more likely generated rather than lost, which indicates that the regarded anomalous events set is more close to the real anomalous events set and contains less normal events with small probability.
This has a positive impact on the anomaly detection.
\end{rem}

\section{Experiments}

Now, we present experimental results to show the efficiency of MIM-based GAN and other classical GANs.
Particularly, we compare our method with other adversarial networks (such as the original GAN, LSGAN and WGAN) with respect to data generation and anomaly detection.
Our main findings for the MIM-based GAN may be that:

\begin{itemize}
\item  Compared with some classical GANs, training performance improvements are available during the training process to the equilibrium;

\item By use of MIM-based GAN, there exists the better performance on detecting anomalies than other classical GANs.
\end{itemize}

\subsection{Datasets}
As for the datasets, artificial data and real data are considered to compare different approaches and evaluate the their performance.
Particularly, on one hand, we take artificial Gaussian distribution data as an example, whose mean and standard deviation are denoted by $\mu$ and $\sigma$ respectively.
The Gaussian distribution $\mathcal{N}$($\mu$, $\sigma$) is chosen as $\mathcal{N}$($\mu=4$, $\sigma=1.25$) in our experiments. %or $\mathcal{N}$($\mu=3$, $\sigma=1$).
On the other hand, several real datasets (including the MNIST databset and Outlier Detection DataSet (ODDS)) are collected online, whose details are listed as follows.

\begin{itemize}
\item \textbf{MNIST}:
As for this dataset, $10$ different classes of digits $\{0,1,2,...,9\}$ in MNIST are generated by use of GANs.
In order to apply this case into anomaly detection, we choose one kind of digit class (such as ``$0$'') as the rare events (namely anomalies), while the rest parts are treated as normal ones.
In other words, there exist $10\%$ rare events mixed in the whole dataset.
The training set consists of $60,000$ image samples ($28 \times 28$ pixel gray handwritten digital images), while there are $10,000$ image samples in the testing set.

\item \textbf{ODDS}:
Considering that it is necessary to process real-world datasets in practice, we shall investigate several anomaly detection datasets from the ODDS repository as follows.

\textit{a) Cardiotocography:}
The Cardiotocography dataset in ODDS repository has $21$ features in each sample, such as Fetal Heart Rate (FHR) and Uterine Contraction (UC) features.
There are $1,831$ samples in this dataset, including $176$ samples ($9.6\%$ data) that belong to the pathologic class, namely the anomalous events class.

\textit{b) Thyroid:}
The Thyroid dataset contains $6$ real attributes (namely data dimension).
In this database, there are $3,772$ samples, containing $93$ hyperfunction samples ($2.5\%$ data) which belong to the anomalous events class.
%\textit{Lymphography:}
%The Lymphography dataset is obtained from the ODDS repository, with each sample having 18 features.
%The dataset comprises of 148 samples, including 6 anomalies ($4.0\%$ contamination).

\textit{c) Musk:}
The Musk dataset in ODDS repository consists of $3,062$ samples including $3.2\%$ anomalies.
Specifically, the dataset contains several-musks and non-musk classes, which are regarded as the inliers and outliers (or anomalies) respectively. By the way, there exist $166$ features in each sample.

\end{itemize}
%\begin{figure}[!t]
%\centering
%\subfigure[ROC curve of different divergences]{\includegraphics[width=3.0in]{ROC.eps}%
%\label{fig_ROC of different divergences}}
%\hfil
%\subfigure[AUC of different divergences]{\includegraphics[width=3.0in]{AUC_box.eps}%
%\label{fig_AUC_box}}
%\hfil
%\subfigure[F-score of different divergences]{\includegraphics[width=3.0in]{Fscore_box.eps}%
%\label{fig_Fscore_box}}
%\caption{Performance of different divergences in the example with each sequence size $\varGamma_0=6000$, sequences number $T_0=200$, outlier sequences number $k_0=20$ and the number of experiments $N_{T_0}=100$. }
%\label{fig_performance}
%\end{figure}

\subsection{Experiment details}

\subsubsection{Training performance experiments}\ \par
To intuitively give some comparisons on the training performances of GANs, we use the artificial data following Gaussian distribution $\mathcal{N}$($\mu=4$, $\sigma=1.25$) to train the MIM-based GAN, original GAN, LSGAN and WGAN.

In particular, we first train the adversarial networks for a number of iterations (e.g. $500$, $1,000$ and $1,500$ training iterations) to obtain a not bad discriminator.
During each iteration, there are $16,000$ samples produced by programming as initial input data for the discriminator.
Then, with the fixed discriminator, a generator is trained by reducing the error of objective function in each kind of adversarial networks.
Furthermore, we adopt two Deep Neural Networks (DNNs) as the generator and discriminator respectively, in which the Stochastic Gradient Descent (SGD) optimizer with $0.001$ learning rate is chosen and the activation functions of discriminator and generator are sigmoid and tanh function respectively.
Finally, we draw the error curves of the objection functions of generators to show the different performance during the training process.

%\subsubsection{Date generation experiment of different GANs}\ \par
%Since the essence of GANs is to generate the data approximating real data, data generation is a representative application to compare different GANs.
%To do this, we use the GANs (including original GAN, LSGAN, MIM-based GAN and WGAN) to generate the data close to the artificial Gaussian distribution data (which is viewed as the real data).
%
%Specifically, we take the data drawn from the Gaussian distribution $\mathcal{N}$($\mu=4$, $\sigma=1.25$) and $\mathcal{N}$($\mu=3$, $\sigma=1$) as the artificial data.
%The GANs are trained for $3000$ iterations and during each iteration, there are $16,000$ Gaussian samples as real data for the discriminator.
%Moreover, two DNNs are chosen as the discriminator and generator whose activation functions are sigmoid and tanh function respectively.
%As well, they both use the SGD optimizer with learning rate $0.001$ in the backpropagation process.
%After training, the curves of empirical Cumulative Distribution Function (CDF) of generative data are drawn to compare different GANs.
%In addition, to evaluate the generative data of GANs, statistical test methods are adopted,  including Z test, Chi-Square variance test and Jarque-Bera test, whose significance levels all satisfy $\alpha=0.05$.

\subsubsection{Anomaly detection experiment based on MNIST and ODDS}\ \par
When detecting anomalies in the MNIST and ODDS, we adopt the general framework of GANs described in Section \ref{section.procedure_detection_GAN}.
However, to compare the efficiency of the GANs, including the original GAN, MIM-based GAN, LSGAN and WGAN, we change the main component (namely different objective functions for the data generation) of the framework.
%In the training and testing process, both normal and rare events are used
The Receiver Operating Characteristic (ROC) curve, Area Under Curve (AUC) and $F_{1}$-score are used as criterions to compare different GANs for anomaly detection in the above datasets.

In details, the detection procedure on the MNIST and ODDS is similar to that mentioned in Section \ref{section.procedure_detection_GAN}, where the neural networks training of GANs is the key point.
At first, as for the framework of neural networks, DNNs are used, whose activation functions in the output layers are sigmoid function and tanh function for the discriminator and generator respectively and those in the other layers are all leaky ReLU.
Moreover, the Adam algorithm is chosen to optimize the weights for the networks with $0.001$ learning rate.
According to the properties of MNIST and ODDS, we configure the hidden layer size $256$ for the MNIST and Musk dataset (belonging to the ODDS) as well as $64$ for the Cardiotocography and Thyroid in the ODDS.
Furthermore, when we obtain trained GANs and the corresponding generative data, we use
Eq. (\ref{eq.Lerror}) and Eq. (\ref{eq.anomaly_score}) (where $\lambda=0.1$ and $\eta=0.05$) to gain anomaly scores for testing data, in which the optimal ${\bm z}_{\text{opt}}$ in latent space is obtained by used of Adam optimizer with learning rate $0.003$.
At last, we adopt Eq. (\ref{eq.decision}) to label the testing data so that the outliers (namely anomalies) are detected.

\begin{figure}[hbpt]
\centering
\includegraphics[width=6.3in]{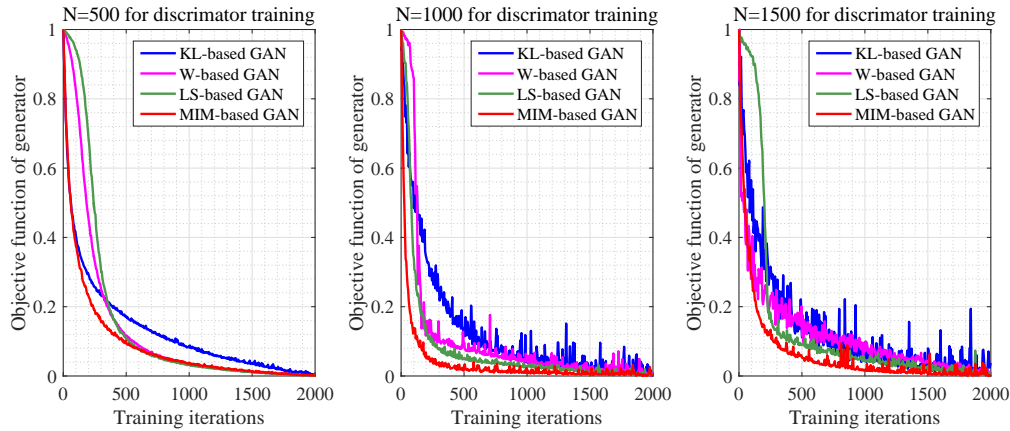}
\caption{The MIM-based GAN, KL-based GAN (namely original GAN), LS-based GAN (namely LSGAN) and W-based GAN (namely WGAN) are used to generate the data following Gaussian distribution
$\mathcal{N}$($\mu=4$, $\sigma=1.25$) where $\mu$ and $\sigma$ denote the mean value and the standard deviation value respectively.
There are $16,000$ input samples following the Gaussian distribution in every training iteration.
The generators are trained with the discriminators fixed for $500$, $1,000$ or $1,500$ training iterations (namely ${\rm N}=500, 1,000 \text{ or } 1,500$).
In these cases, the curves for the objection functions of generators are drawn to compare the performances on training process.
}
\label{fig_gradient}
\end{figure}

\subsection{Results and discussion}
According to the above design for experiments, we do simulations with Pytorch and Matlab to evaluate the theoretical analysis. In particular, the experiment results are discussed as follows.

\subsubsection{Training performance of GANs}\ \par
From figure \ref{fig_gradient}, it is illustrated that during the training process,
the MIM-based GAN, LSGAN and WGAN all perform better than the original GAN in the aspects of convergence and stability of training process.
While, the MIM-based GAN also has its own superiority.
%%================================================================================
On one hand, the MIM-based GAN has better convergence than the other GANs with respect to the objective function of generator.
On the other hand, when a discriminator is given, the generator for MIM-based GAN has more stability than that for original GAN, which is also comparable to those for LSGAN and WGAN.
In brief, the MIM-based GAN converges to equilibrium faster and more stably than the original GAN. It also has these advantages on the training process to some degree, compared with LSGAN and WGAN.

%\subsubsection{Results of data generation of different GANs}\ \par
%From figure \ref{fig_Gaussian_generation}, it is shown that the GANs based on different information measurements (including KL-based GAN, MIM-based GAN, LS-based GAN and W-based GAN) all generate the data similarly following the true normal Gaussian distribution.
%%%================================================================================
%As well, from the Table \ref{table.test_GANs}, it is not difficult to see that different GANs have the same decision for the hypothesis in each statistic test.
%In Table \ref{table.test_GANs}, $H=0$ indicates a failure to reject the null hypothesis at the significance level and $P$-value is the probability value to accept the null hypothesis.
%By virtue of the test results, we accept that the generative data from different GANs all follows the target Gaussian distribution.
%Although these GANs have different advantageous performance under different statistic test methods, this makes no difference to the test results.
%Consequently, MIM-based GAN has the same ability as the other classical GANs to generate Gaussian distribution data.

\begin{figure}[hbpt]
\centering
\includegraphics[width=6.3in]{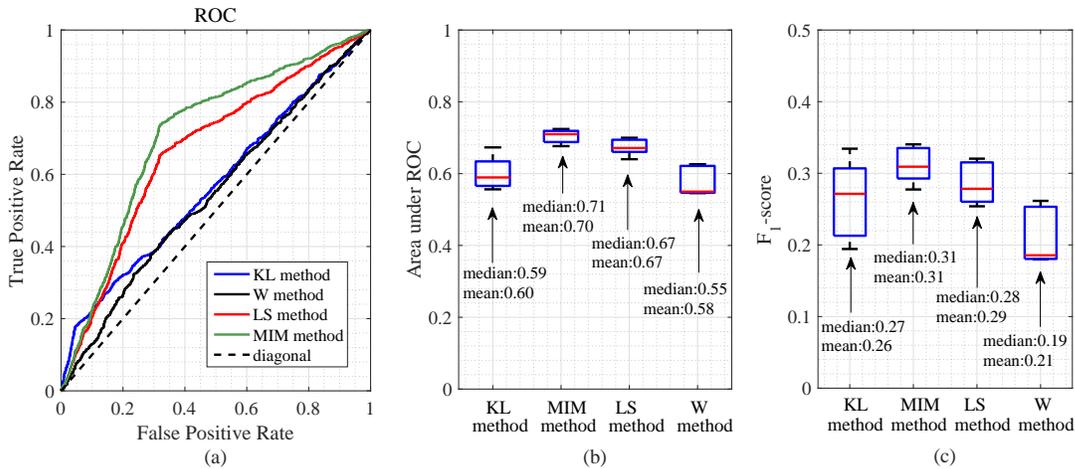}
\caption{ROC curve, AUC and $F_{1}$-score for GAN-based anomaly detection in MNIST dataset, where  MIM-based GAN (corresponding to MIM method), original GAN (corresponding to KL method), LSGAN (corresponding to LS method) and WGAN (corresponding to W method) are used to generate data.
%The MIM-based GAN (corresponding to MIM method), original GAN (corresponding to KL method), LSGAN (corresponding to LS method) and WGAN (corresponding to W method) are used to generate the data similar to the MNIST. %with $30$ training iterations.
%The samples labeled digit ``$0$'' in the MNIST are chosen as the anomalous events where there are $60,000$ training samples and $10,000$ testing samples in the MNIST.
%By using the framework described in Section \ref{section.procedure_detection_GAN} to detect the anomalous events in the testing dataset, the ROC curve of detection results is shown (as subfigure ($a$)), as well as, AUC and $F_{1}$-score are shown (as subfigures ($b$) and ($c$)) with $20$ experiments to provide the performance comparison for the different GAN-based detection approaches.
}
\label{fig_mnist}
\end{figure}
\subsubsection{Results of anomaly detection for MNIST and ODDS}\ \par
Figure \ref{fig_mnist} shows the performance of anomaly detection with different GANs in the MNIST experiment, including ROC curve, AUC and $F_{1}$-score.
In general, it is not difficult to see that the MIM-based GAN improves these kinds of performance compared with the other GANs.
This is resulted from the fact that MIM-based objective function enlarges the proportion of rare events.
Moreover, we also see that the volatility of the detection results (shown by the AUC and $F_{1}$-score) in the KL-based GAN and W-based GAN methods is greater than that in two other methods.

Figure \ref{fig_cardio}, \ref{fig_thyroid} and \ref{fig_musk} compare several GAN-based detection methods in the case of ODDS by showing the performances in terms of ROC curve, AUC and $F_1$-score.
Although there exists different performance in the three datasets of ODDS, MIM-based GAN still performs better on the anomaly detection.
Particularly, on one hand, the ROC curve for MIM-based GAN method is more close to the ideal than those for the other detection methods with GANs.
On the other hand, as for the AUC and $F_1$-score of detection results, the corresponding statistics (including the median and mean) have better results when using MIM-based GAN method as shown in the box-plots ($b$) and ($c$) in Figure \ref{fig_cardio}, \ref{fig_thyroid} and \ref{fig_musk}.
Actually, these datasets from ODDS repository provide some convincing and realistic evidences for the advantage of MIM-based GAN on anomaly detection.
%Moreover, we see that original GAN has similar performance to the other GANs.
\begin{figure}[hbpt]
\centering
\includegraphics[width=6.3in]{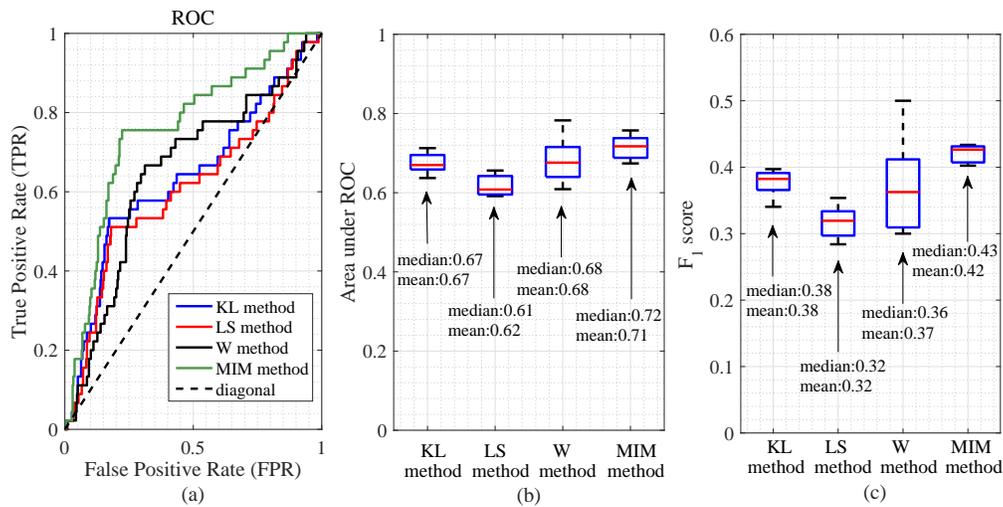}
\caption{ROC curve, AUC and $F_{1}$-score for GAN-based anomaly detection in Cardiotocography dataset.
%The MIM-based GAN, original GAN, LSGAN and WGAN are used to generate the data similar to the Cardiotocography dataset (which belongs to the ODDS).
%After randomly shuffling the data, we use $1,360$ samples (about $74.28 \%$ data) as the input samples to train the GANs (with $500$ training iterations) and the rest are treated as the testing samples for the detection procedure.
%In order to compare the performance of the different GAN-based approaches, the ROC curve, AUC and $F_{1}$-score are drawn (as subfigures ($a$), ($b$) and ($c$)) with $20$ experiments in this figure.
}
\label{fig_cardio}
\end{figure}

\begin{figure}[hbpt]
\centering
\includegraphics[width=6.3in]{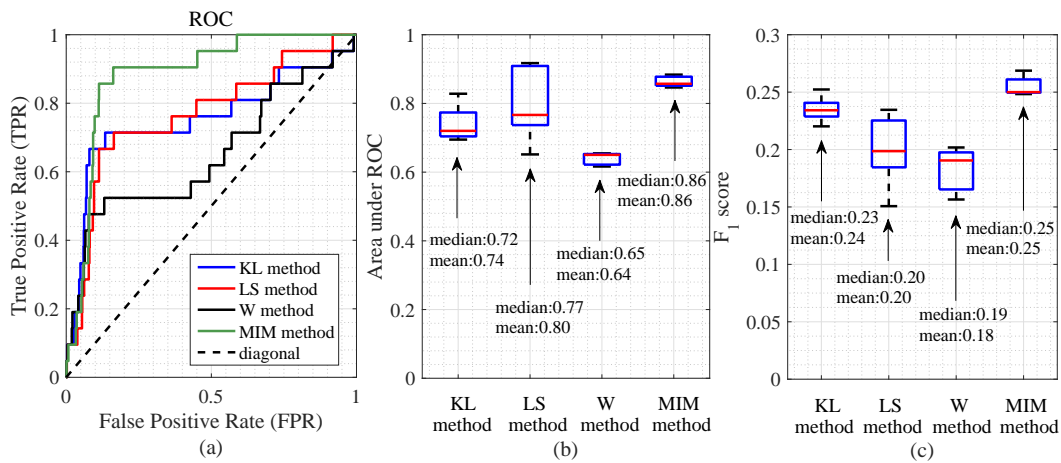}
\caption{ROC curve, AUC and $F_{1}$-score for GAN-based anomaly detection in Thyroid dataset.
%Considering the anomaly detection on the Thyroid dataset (a kind of dataset in the ODDS), we use $2800$ samples (about $74.23 \%$ data) as training samples for the the MIM-based GAN, original GAN, LSGAN and WGAN (whose training iterations are all $500$).
%Then, the rest samples are treated as the testing samples.
%To intuitively compare the detection results of different GAN-based methods, the ROC curve, AUC and $F_{1}$-score are drawn (as box-plots ($a$), ($b$) and ($c$)) with $20$ experiments.
}
\label{fig_thyroid}
\end{figure}

\begin{figure}[!t]
\centering
\includegraphics[width=6.3in]{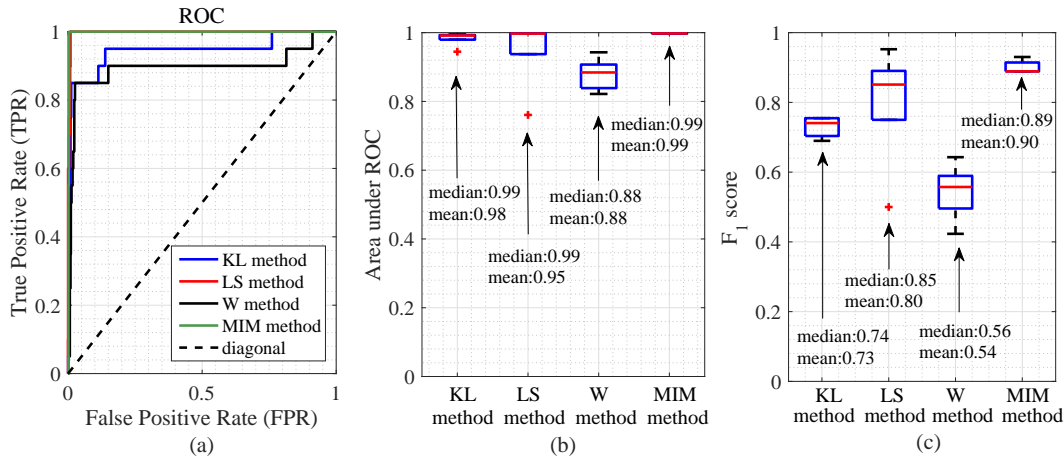}
\caption{ROC curve, AUC and $F_{1}$-score for GAN-based anomaly detection in Musk dataset.
%To generate data similar to the Musk dataset, the MIM-based GAN, original GAN, LSGAN and WGAN are used with $2,200$ training samples (about $71.85 \%$ data)
%where the number of training iterations is $20$.
%The rest samples are regarded as testing samples for these GAN-based detection methods.
%To show detection results intuitively, the ROC curve is given (as subfigure ($a$)), as well as, AUC and $F_{1}$-score are drawn (as subfigures ($b$) and ($c$)) with $20$ experiments in this figure.
}
\label{fig_musk}
\end{figure}

\section{Conclusion}
In this paper, we proposed a new model named MIM-based GAN to enrich the conventional GANs.
In terms of this model, it has different performance of training process with the other classical GANs.
Furthermore, another advantage of this new developed approach is to highlight the proportion of rare events in the objective function to generate more efficient data.
In addition, we showed that compared with the classical GANs, the MIM-based GAN has more superiority on the anomaly detection, which may be a promising application direction with GANs in practice.

\section*{Acknowledgment}
We thank a lot for the advices from Prof. Xiaodong Wang and Xiao-Yang Liu at Columbia Uiversity, USA. Their suggestions make a positive effect on this work.

%\newpage
\appendices
%\section*{Appendix}
%\section{Proof of Lemma \ref{lem.Optimality}}\label{app.lem.Optimality}

%\section{Proof of Proposition \ref{prop.maximum}}\label{app.prop.maximum}

%\section{Proof of Corollary \ref{cor.sensitivity_D}}\label{app.prop_sensitivity_D}

%\section{Proof of Proposition \ref{prop.LMIM_sensibility}}\label{app.prop.LMIM_sensibility}

%\section{Proof of Corollary \ref{cor.gradient}}\label{app.cor_gradient}
%%%According to , it is easy to see that
%%%\begin{equation}\label{eq.LMIM_gradient0}
%%%\begin{aligned}
%%%    & \nabla_{\bm\theta} \mathbb{E}_{{\bm z}\sim \mathbb{P}_z} [\exp(D^{*}_{\text{MIM}}(g_{\bm\theta}({\bm z}))) ] \\
%%%    & = \frac{\sqrt{{\rm e}}}{2} \mathbb{E}_{{\bm z}\sim \mathbb{P}_z}
%%%    \Big[
%%%    \frac{ \nabla_{\bm\theta}P(g_{\bm\theta}({\bm z}))
%%%    \sqrt{\frac{P_{g_{\theta}}(g_{\bm\theta}({\bm z}))}{P(g_{\bm\theta}({\bm z}))}}
%%%    - \nabla_{\bm\theta}P_{g_{\theta}}(g_{\bm\theta}({\bm z}))
%%%    \sqrt{\frac{P(g_{\bm\theta}({\bm z}))}{P_{g_{\theta}}(g_{\bm\theta}({\bm z}))}} }
%%%    {P_{g_{\theta}}(g_{\bm\theta}({\bm z}))}
%%%    \Big]\text{.}\\
%%%\end{aligned}
%%%\end{equation}

%\section{Proof of Corollary \ref{cor.convergence_rate}}\label{app.cor_convergence_rate}

%\subsection*{Proof of Proposition \ref{prop.stability}}
\section{Proof of Corollary \ref{cor.stability}}\label{app.cor_stability}

With respect to the MIM-based GAN, original GAN and its improved GAN, it is not difficult to see that the corresponding gradient functions of generators depend on
$\nabla_{\bm\theta}\mathbb{E}_{{\bm z}\sim \mathbb{P}_z} [\exp(D(g_{\bm\theta}({\bm z}))) ]$,
$\nabla_{\bm\theta}\mathbb{E}_{{\bm z}\sim \mathbb{P}_z} [\ln (1-D(g_{\bm\theta}({\bm z}))) ]$ and
$\nabla_{\bm\theta}\mathbb{E}_{{\bm z}\sim \mathbb{P}_z} [-\ln (D(g_{\bm\theta}({\bm z}))) ]$, respectively.

On one hand, in the case that $D-\tilde D^*=\epsilon$ ($\epsilon\in [0,1]$ and $\tilde D^*(g_{\bm\theta}({\bm z}))=0$), we have
\begin{equation}\label{eq.KL1_perfect_stability}
\begin{aligned}
     \nabla_{\bm\theta}\mathbb{E}_{{\bm z}\sim \mathbb{P}_z} [\ln (1-D(g_{\bm\theta}({\bm z}))) ]
    & = \mathbb{E}_{{\bm z}\sim \mathbb{P}_z}
    [- \frac{ \nabla_{{\bm x}}D({\bm x})\nabla_{\bm\theta}g_{\bm\theta}({\bm z})}
    {1-D(g_{\bm\theta}({\bm z}))} ]\\
    & = \mathbb{E}_{{\bm z}\sim \mathbb{P}_z}
    [- \frac{ \nabla_{{\bm x}}D({\bm x})\nabla_{\bm\theta}g_{\bm\theta}({\bm z})}
    {1-\tilde D^*(g_{\bm\theta}({\bm z})) -\epsilon} ]\\
    & = -\frac{1}{1 -\epsilon}
    \mathbb{E}_{{\bm z}\sim \mathbb{P}_z}[{ \nabla_{{\bm x}}D({\bm x})\nabla_{\bm\theta}g_{\bm\theta}({\bm z})}
     ]\text{,}
\end{aligned}
\end{equation}
\begin{equation}\label{eq.KL2_perfect_stability}
\begin{aligned}
     \nabla_{\bm\theta}\mathbb{E}_{{\bm z}\sim \mathbb{P}_z} [-\ln (D(g_{\bm\theta}({\bm z}))) ]
    & = \mathbb{E}_{{\bm z}\sim \mathbb{P}_z}
    [- \frac{\nabla_{{\bm x}}D({\bm x})\nabla_{\bm\theta}g_{\bm\theta}({\bm z}) }
    {D(g_{\bm\theta}({\bm z}))} ]\\
    & = \mathbb{E}_{{\bm z}\sim \mathbb{P}_z}
    [- \frac{\nabla_{{\bm x}}D({\bm x})\nabla_{\bm\theta}g_{\bm\theta}({\bm z}) }
    {\tilde D^*(g_{\bm\theta}({\bm z})) +\epsilon} ]\\
    & = -\frac{1}{\epsilon}
    \mathbb{E}_{{\bm z}\sim \mathbb{P}_z}[{ \nabla_{{\bm x}}D({\bm x})\nabla_{\bm\theta}g_{\bm\theta}({\bm z})}
     ]\text{.}
\end{aligned}
\end{equation}

It is not difficult to see that $\frac{1}{\epsilon}$ and $\frac{1}{1-\epsilon}$ are symmetric in the case $\epsilon\in [0,1]$, which implies that $\nabla_{\bm\theta}\mathbb{E}_{{\bm z}\sim \mathbb{P}_z} [-\ln (D(g_{\bm\theta}({\bm z}))) ]$ has the same anti-interference ability as $\nabla_{\bm\theta}\mathbb{E}_{{\bm z}\sim \mathbb{P}_z} [\ln (1-D(g_{\bm\theta}({\bm z}))) ]$.

As for the function $h(u)=\frac{1}{1-u}-\exp(u)$ ($u\in [0,1]$), it is not difficult to see that
$h(u)\ge 0$ for $u\in [0,1]$ where the equation holds at $u=0$. In fact, $\exp(-u)\geq 1-u$ for $u\in [0,1]$.
Therefore, by comparing Eq. (\ref{eq.MIM_perfect_stability}) with Eq. (\ref{eq.KL1_perfect_stability}) and Eq. (\ref{eq.KL2_perfect_stability}),
it is readily seen that the result of this corollary is true.

On the other hand, in the case that $D-\check D^*=\epsilon$ ($|\epsilon|<\frac{1}{2}$ and $\check D^*(g_{\bm\theta}({\bm z}))=\frac{1}{2}$), it is not difficult to see that
\begin{equation}\label{eq.KL1_worst_stability}
\begin{aligned}
     \nabla_{\bm\theta}\mathbb{E}_{{\bm z}\sim \mathbb{P}_z} [\ln (1-D(g_{\bm\theta}({\bm z}))) ]
    & = \mathbb{E}_{{\bm z}\sim \mathbb{P}_z}
    [- \frac{ \nabla_{\bm x}D({\bm x})\nabla_{\bm\theta}g_{\bm\theta}({\bm z})}
    {1-\check D^*(g_{\bm\theta}({\bm z})) -\epsilon} ]\\
    & = -\frac{1}{\frac{1}{2}-\epsilon}
    \mathbb{E}_{{\bm z}\sim \mathbb{P}_z}[{\nabla_{{\bm x}}D({\bm x})\nabla_{\bm\theta}g_{\bm\theta}({\bm z}) }
     ]\text{,}\\
\end{aligned}
\end{equation}

\begin{equation}\label{eq.KL2_worst_stability}
\begin{aligned}
     \nabla_{\bm\theta}\mathbb{E}_{{\bm z}\sim \mathbb{P}_z} [-\ln (D(g_{\bm\theta}({\bm z}))) ]
    & = \mathbb{E}_{{\bm z}\sim \mathbb{P}_z}
    [- \frac{\nabla_{\bm x}D({\bm x}) \nabla_{\bm\theta}g_{\bm\theta}({\bm z}) }
    {\check D^*(g_{\bm\theta}({\bm z})) +\epsilon} ]\\
    & = -\frac{1}{\frac{1}{2}+\epsilon}
    \mathbb{E}_{{\bm z}\sim \mathbb{P}_z}[{ \nabla_{\bm x}D({\bm x})\nabla_{\bm\theta}g_{\bm\theta}({\bm z})}
     ]\text{.}\\
\end{aligned}
\end{equation}

It is readily seen that $\frac{1}{\frac{1}{2}-\epsilon}$ and $\frac{1}{\frac{1}{2}+\epsilon}$ are symmetric if there exists a small disturbance, i.e. $|\epsilon|< C< \frac{1}{2}$ ($C$ is a constant), that is to say, the anti-interference ability for $\nabla_{\bm\theta}\mathbb{E}_{{\bm z}\sim \mathbb{P}_z} [-\ln (D(g_{\bm\theta}({\bm z}))) ]$  is equal to that for $\nabla_{\bm\theta}\mathbb{E}_{{\bm z}\sim \mathbb{P}_z} [\ln (1-D(g_{\bm\theta}({\bm z}))) ]$.

Considering the function $h(u)=\frac{1}{\frac{1}{2}-u}-\exp(\frac{1}{2}+u)$ where $u\in (-\frac{1}{2},\frac{1}{2})$, it is readily seen that $h(u)> 0$ for $u\in (-\frac{1}{2},\frac{1}{2})$.
Thus, by comparing Eq. (\ref{eq.MIM_worst_stability}) with Eq. (\ref{eq.KL1_worst_stability}) and Eq. (\ref{eq.KL2_worst_stability}),
it is easy to see that the result of this corollary also is true.

To sum up, the corollary is verified.

\end{document}